\documentclass{article}

\PassOptionsToPackage{numbers, compress}{natbib}
\usepackage{amsmath,amssymb,amsthm,fullpage,mathrsfs,pgf,tikz,caption,subcaption,mathtools}
\usepackage[ruled,vlined,linesnumbered]{algorithm2e}
\usepackage{natbib}
\usepackage{amsmath,amssymb,amsthm,mathtools}
\usepackage[utf8]{inputenc} % allow utf-8 input
\usepackage[T1]{fontenc}    % use 8-bit T1 fonts
\usepackage{hyperref}       % hyperlinks
\usepackage{url}            % simple URL typesetting
\usepackage{booktabs}       % professional-quality tables
\usepackage{amsfonts}       % blackboard math symbols
\usepackage{nicefrac}       % compact symbols for 1/2, etc.
\usepackage{microtype}      % microtypography
\usepackage[margin=1in]{geometry}
\usepackage{bbm}
\usepackage{enumitem}
\usepackage{xcolor}
\usepackage{cleveref}
\usepackage{caption}
\usepackage{subcaption}

\let\oldnl\nl
\newcommand{\nonl}{\renewcommand{\nl}{\let\nl\oldnl}}

\RestyleAlgo{boxruled}

\newtheorem{theorem}{Theorem}[section]
\newtheorem{corollary}[theorem]{Corollary}
\newtheorem{proposition}[theorem]{Proposition}

\newtheorem{definition}[theorem]{Definition}

\newcommand{\bigo}{\mathcal{O}}

\title{Bounded Memory Active Learning through Enriched Queries}

\author{%
  Max~Hopkins\thanks{Department of Computer Science and Engineering, UCSD, CA 92092. Email: \texttt{nmhopkin@eng.ucsd.edu}. Supported by NSF Award DGE-1650112},
  Daniel~Kane\thanks{Department of Computer Science and Engineering / Department of Mathematics, UCSD, CA 92092. Email: \texttt{dakane@eng.ucsd.edu}. Supported by NSF CAREER Award ID 1553288 and a Sloan fellowship},
    Shachar~Lovett\thanks{Department of Computer Science and Engineering, UCSD, CA 92092. Email: \texttt{slovett@cs.ucsd.edu}. Supported by NSF CAREER award 1350481, CCF award 1614023 and a
Sloan fellowship},
Michal~Moshkovitz\thanks{Qualcomm Institute, UCSD, California, CA 92092. Emaili: mmoshkovitz@eng.ucsd.edu}
}

\begin{document}

\maketitle
\begin{abstract}
    The explosive growth of easily-accessible unlabeled data has lead to growing interest in \emph{active learning}, a paradigm in which data-hungry learning algorithms adaptively select informative examples in order to lower prohibitively expensive labeling costs. Unfortunately, in standard worst-case models of learning, the active setting often provides no improvement over non-adaptive algorithms. To combat this, a series of recent works have considered a model in which the learner may ask \emph{enriched} queries beyond labels. While such models have seen success in drastically lowering label costs, they tend to come at the expense of requiring large amounts of memory. In this work, we study what families of classifiers can be learned in \emph{bounded memory}. To this end, we introduce a novel streaming-variant of enriched-query active learning along with a natural combinatorial parameter called \emph{lossless sample compression} that is sufficient for learning not only with bounded memory, but in a query-optimal and computationally efficient manner as well. Finally, we give three fundamental examples of classifier families with small, easy to compute lossless compression schemes when given access to basic enriched queries: axis-aligned rectangles, decision trees, and halfspaces in two dimensions.
\end{abstract}
\section{Introduction}\label{sec:intro}
    Today's learning landscape is dominated mostly by data-hungry algorithms, each requiring a massive supply of labeled training samples in order to reach state of the art accuracy. Such algorithms are excellent when labeled data is cheap and plentiful, but in many important scenarios acquiring labels requires the use of human experts, making popular supervised methods like deep learning infeasible both in time and cost. In recent years, a framework meant to address this issue called \textit{active learning} has gained traction in both theory and practice. Active learning posits that not all labeled samples are equal: some may be particularly informative, others useless. While a standard supervised (passive) learning algorithm receives a stream or pool of labeled training data, an active learner instead receives \textit{unlabeled} data along with the ability to query an expert \textit{labeling oracle}. By choosing only to query the most informative examples, the hope is that an active learner can achieve state of the art accuracy using only a small fraction of the labels required by passive techniques.
    
    While active learning saw initial success with simple classifiers such as thresholds in one-dimension, it quickly became clear that inherent structural barriers barred it from improving substantially over the passive case even for very basic examples such as halfspaces in two-dimensions or axis-aligned rectangles \cite{dasgupta2005analysis}. A number of modifications to the model have been proposed to remedy this issue. One such strategy that has gained increasing traction in the past few years is empowering the learner to ask questions beyond simple label queries. One might ask the oracle, for instace, to \textit{compare} two pieces of data rather than simply label them---the idea being that such additional information might break down the structural barriers inherent in standard lower bounds. Indeed, in 2017, Kane, Lovett, Moran, and Zhang (KLMZ) \cite{kane2017active} showed not only how this was true for halfspaces,\footnote{We note their work requires structural assumptions on the data to work beyond two dimensions.} but introduced a combinatorial complexity parameter called \textit{inference dimension} to characterize exactly when a family of classifiers is efficiently actively learnable with respect to some set of enriched queries.
    
    The model proposed by KLMZ \cite{kane2017active}, however, is not without its downsides. One issue is that the model is \textit{pool-based}, meaning the algorithm receives a large pool of unlabeled samples ahead of time, and can query or otherwise access any part of the sample at any time. This type of model can be infeasible in practice due to its unrealistic memory requirements---the learner is assumed to always have the full training data in storage. In this work, we aim to resolve this issue by studying when a family of classifiers can be efficiently active learned in \textit{bounded memory}, meaning the amount of memory used by the algorithm should remain \textit{constant} regardless of the desired accuracy. Such algorithms open up potential applications of active learning to scenarios where storage is severely limited, e.g.\ to smartphones and other mobile devices.
    
    To this end, we introduce a new \textit{streaming} variant of active learning with enriched queries in which the learner has access to a stream of unlabeled data and chooses one-by-one whether to store or forget points from the stream. Instead of having query access to the full training set at any time, our algorithm is then restricted to querying only points it has stored in memory. Along with this model, we introduce a natural strengthening of Littlestone and Warmuth's \cite{littlestone1986relating} \textit{sample compression}, a standard learning technique for Valiant's Probably Approximately Correct (PAC) \cite{valiant1984theory,vapnik1974theory} model, called \textit{lossless} sample compression, and show that any class with such a scheme may be learned query-optimally and with constant memory. In doing so, we make the first non-trivial advance towards answering an open question posed by KLMZ \cite{kane2017active} regarding the existence of a combinatorial characterization for bounded-memory active learning. Further, we show that lossless compression schemes imply learnability not only in the standard PAC-model, but also in a much stronger sense known as \textit{Reliable and Probably Useful Learning} \cite{rivest1988learning}. This model, which carries strong connections to the active learning paradigm \cite{el2012active,kane2017active}, demands that the learner be perfect (makes no errors) with the caveat that it may abstain from classifying a small fraction of examples.
    
    Finally, we conclude by showing that a number of classifier families fundamental to machine learning exhibit small lossless compression schemes with respect to natural enriched queries. We focus in particular on three such classes: axis-aligned rectangles, decision trees, and halfspaces in two dimensions. In each of these three cases our lossless compression scheme is efficiently computable, resulting in computationally efficient as well as query-optimal and bounded memory learners. All three classes provide powerful examples of how natural enriched queries can turn fundamental learning problems from prohibitively expensive to surprisingly feasible.
\subsection{Main Results}
We start with an informal overview of our contributions: namely the introduction of lossless sample compression, its implications for efficient, bounded memory RPU-learning, and three fundamental examples of classes which, while infeasible or even impossible in standard models, have small lossless compression schemes with respect to natural enriched queries. Before launching into such results, however, we give a brief introduction to enriched queries, followed by some intuition and background on our novel form of compression. 

Given a set $X$ and a binary classifier $h: X \to \{0,1\}$, standard models of learning generally aim to approximate $h$ solely through the use of labeled samples from $X$. Since labels often cannot provide enough information to learn efficiently, we allow the learner to ask some specified set of additional questions, denoted by a ``query set'' $Q$ (see \Cref{sec:queries} for formal description). As an example, one well-studied notion of an enriched query is a ``comparison'' \cite{jamieson2011active,karbasi2012comparison, wauthier2012active,xu2017noise,hopkins2020power,hopkins2020noise,Cui2020uncertainty}. In such cases, along with asking for labels, the learner may additionally ask an expert to \textit{compare} two pieces of data (e.g. asking a doctor ``which of these patients do you think is sicker?''). Given a sample $S \subseteq X$, we let $Q_h(S)$ then denote the responses to all possible queries in $Q$ on $S$ under hypothesis $h$. For the basic example of labels and comparisons, this would consist of $|S|$ labels and all ${|S| \choose 2}$ pairwise comparisons.

With this notion in hand, we can discuss our extension of sample compression to the enriched query regime. Standard sample compression schemes posit the existence of a compression algorithm $A$ and decompression scheme $D$ such that for any sample $S$, $A(S)$ is small, and $D(A(S))$ outputs a hypothesis that correctly labels all elements of $S$. Lossless sample compression strengthens this idea in two ways. First, the hypothesis output by $D$ must be zero-error (but, like our learners, is allowed to abstain). Second, the output hypothesis must label not just $S$, but every point whose label can be inferred\footnote{The label of $x \in X$ is inferred by a subset $S \subseteq X$ if all concepts $h \in H$ consistent with queries on $S$ share the same label for $x$. See \Cref{sec:infer} for details.} from $Q_h(S)$. 
\begin{definition}[Informal \Cref{def:LCS}]\label{def:intro-LCS}
Let $X$ be a set and $H$ a family of binary classifiers on $X$. We say $(X,H)$ has a lossless compression scheme (LCS) $W$ of size $k$ with respect to a set of enriched queries $Q$ if for all classifiers $h \in H$ and subsets $S \subset X$, there exists a subset $W = W(Q_h(S)) \subseteq S$ such that $|W| \leq k$, and any point in $X$ whose label is inferred by $Q_h(S)$ is also inferred by queries on $Q_h(W)$.
\end{definition}

Given the existence of a lossless compression scheme for some class $(X,H)$, we prove in \Cref{sec:LCS} that (a slight variant of) the following simple algorithm learns $(X,H)$ query-optimally, in bounded memory, and with no errors.\footnote{As mentioned previously, the learner is allowed to say ``I don't know'' on a small fraction of the space. See \Cref{sec:RPU} for formal definition.}
\\
\\
\begin{algorithm}[H]

\KwResult{Returns a zero-error classifier that labels a $1-\varepsilon$ fraction of $X$ with probability $1-\delta$.}
\nonl \textbf{Input:} Query set $Q$, hypothesis class $(X,H)$, sample oracle $\bigo_X$, and lossless compression scheme $W$.\\
\nonl \textbf{Parameters:} 
\begin{itemize}
    \item Size of LCS $k$
    \item Query cap $T_1=O\left(\log\left(\frac{1}{\varepsilon\delta}\right)\right)$
    \item Sample cap $T_2=\tilde{O}\left(\frac{k\log^2(1/\delta)}{\varepsilon}\right)$
\end{itemize}
\nonl \textbf{Algorithm:}\\
\textit{Initialize} $i=0$, $j=0$, $C_0=\{\varnothing\}$, $X_0=X$\;
\textit{While} $i \leq T_1$:
\begin{enumerate}
    \item Sample a subset $S_i \subseteq X_i$ of size $6k$ (via rejection sampling on $\bigo_X$).
    \begin{enumerate}
        \item For each point drawn from $X$ in this process, increment $j$.
        \item If $j$ reaches $T_2$, abort and return labels inferred by queries on $C_i$
    \end{enumerate}
    \item Make all queries on $S_i \cup C_i$, and compute $C_{i+1}=W(Q_h(S_i \cup C_i))$
    \item Remove all points in and queries on $(S_i \cup C_i) \setminus C_{i+1}$ from memory and increment $i$.
    \item Set $X_i \subseteq X$ to be the set of points uninferred by queries on $C_i$
\end{enumerate}
\textit{Return} labels inferred by queries on $C_i$
 \caption{Bounded Memory RPU-Learning via Lossless Compression}
 \label{alg:intro}
 
\end{algorithm}
It is worth noting that in \Cref{alg:intro}, the set of remaining uninferred points $X_i$ need not be kept in memory. Membership in $X_i$ can be checked in an online fashion in step 1. We prove in \Cref{sec:LCS} that \Cref{alg:intro} is a query and computationally efficient, bounded memory RPU-learner.
\begin{theorem}[Informal \Cref{cor:bounded}]\label{thm:intro-bounded}
\Cref{alg:intro} actively RPU-learns $(X,H)$ in only 
\[
q(\varepsilon,\delta) \leq O_k\left(\log(1/\varepsilon)\right)
\]
queries,
\[
T(\varepsilon,\delta) \leq \tilde{O}_k\left(T_{X,H,W}\frac{\log^2(1/\delta)}{\varepsilon}\right)
\]
time, and 
\[
M(X,H) \leq O_k(1)
\]
memory, where we have suppressed dependence on $k$ (the size of the LCS), and $T_{X,H,W}$ is a parameter dependent only on the class and compression scheme $W$. In all examples we study $T_{X,H,W}$ is small and dependence on $k$ is at worst quadratic.
\end{theorem}
 It is further  worth noting that $O(\log(1/\varepsilon))$ query complexity is information-theoretically optimal for most non-trivial concept classes. As long as the class has $\text{poly}(1/\varepsilon)$ concepts which are $\Omega(\varepsilon)$  separated, any active learner must make $\Omega(\log(1/\varepsilon))$ queries to distinguish between them \cite{kulkarni1993active}.

We give three fundamental examples of classes which are either impossible or highly infeasible to RPU-learn with standard label queries, but have small, efficiently computable lossless compression schemes with respect to natural enriched queries: axis-aligned rectangles, decision trees, and halfspaces in two dimensions. As a result, these classes are all efficiently RPU-learnable with bounded memory. We briefly introduce each class, our proposed enriched queries, and discuss the implications on their learnability. We start our discussion with the simplest of the three, axis-aligned rectangles in $\mathbb{R}^d$, which correspond to indicator functions for products of intervals over $\mathbb{R}$:
\[
R = [a_1,b_1] \times \ldots \times [a_d,b_d]
\]
where $a_i \leq b_i$. While axis-aligned rectangles are impossible to RPU-learn in a finite number of label queries \cite{Kivinen2}, we will show that the class has a small lossless compression scheme with respect to a natural query we call the ``odd-one-out'' oracle $\bigo_{\text{odd}}$. Notice that axis-aligned rectangles essentially define a certain ``acceptable'' range for every feature---a point is labeled $1$ iff it lies in this range for all coordinates. Informally, given a point $x\in\mathbb{R}^d$ lying outside the rectangle, an ``odd-one-out'' query simply asks the user ``why do you dislike $x$?''. Concretely, one might imagine a chef is trying to cook a dish for a particularly picky patron. After each failed attempt, the chef may ask the patron what went wrong---perhaps the patron thinks the meat was overcooked! More formally, the ``odd-one-out'' query asks for a violated coordinate (i.e.\ a feature lying outside the acceptable range), and whether the coordinate was too large (in our example, overcooked) or too small (undercooked).

We prove that axis-aligned rectangles in $\mathbb{R}^d$ have an efficiently computable lossless compression scheme of size $O(d)$, and thus are efficiently RPU-learnable in bounded memory with near-optimal query complexity.
\begin{corollary}[(Informal) \Cref{cor:rect}]\label{cor:intro-rect}
The class of axis-aligned rectangles in $\mathbb{R}^d$ is RPU-learnable in only
\[
q(\varepsilon,\delta) = O\left (d\log \left (\frac{1}{\varepsilon\delta} \right ) \right)
\]
queries, $O(d)$ memory, and time $\tilde{O}\left(\frac{d^2\log^2(1/\delta)}{\varepsilon}\right)$ when the learner has access to $\bigo_{\text{odd}}$.
\end{corollary}
While rectangles provide an excellent theoretical example of a class for which basic enriched queries break standard barriers in active and bounded memory learning, they are often too simple to be of much practical use. To this end, we next consider a broad generalization of the class of axis-aligned rectangles called decision trees. A decision tree over $\mathbb{R}^d$ is a binary tree where each node in the tree corresponds to an inequality:
\[
x_i \overset{?}{\geq} b \ \text{or} \ x_i \overset{?}{\leq} b,
\]
measuring the $i$-th feature (coordinate) of any $x \in \mathbb{R}^d$. Each leaf in the tree is assigned a label, and the label of any $x \in X$ is uniquely determined by the leaf resulting from following the decision tree from root to leaf, always taking the path determined by the inequality at each node. Informally, a decision tree may then be thought of as a partition of $\mathbb{R}^d$ into clusters (axis-aligned rectangle) given by the leaves. Our proposed enriched query for decision trees, the ``same-leaf'' oracle $\bigo_{\text{leaf}}$, builds off this intuition. Given a decision tree $T$ and two points $x,x' \in \mathbb{R}^d$, $\bigo_{\text{leaf}}(T,x,x')$ determines whether $x$ and $x'$ lie in the same leaf of the decision tree. Thinking of each leaf as a cluster, this query may be seen as a variant of the ``same-cluster'' query paradigm studied in many recent works \cite{ashtiani2016clustering,verroios2017waldo,mazumdar2017clustering,ailon2018approximate,firmani2018robust,dasgupta2018learning,saha2019correlation}. For our scenario, think of asking a user whether two movies they like are of the same genre. We prove that the class of decision trees of size $s$ (at most $s$ leaves) has a small, efficiently computable lossless compression scheme and are therefore efficiently learnable in bounded memory.

\begin{corollary}[Informal \Cref{cor:DT-fixed}]\label{cor:intro-DT-fixed}
The class of size $s$ decision trees over $\mathbb{R}^d$ is RPU-learnable in only
\[
q(\varepsilon,\delta) = O\left (ds^2\log \left (\frac{1}{\varepsilon\delta} \right ) \right)
\]
queries, $O(ds)$ memory, and time $\tilde{O}\left(\frac{d^2s^2\log(1/\delta)^2}{\varepsilon}\right)$ when the learner has access to $\bigo_{\text{leaf}}$.
\end{corollary}
It is worth noting that while the class of decision trees of arbitrary size is not learnable (even in the PAC-setting \cite{hancock1996lower}), we can bootstrap \Cref{thm:intro-bounded} and \Cref{cor:intro-DT-fixed} to build an algorithm that learns decision trees \textit{attribute-efficiently}. That is to say an algorithm whose \textit{expected} time, number of queries, and memory scales with the size of the unknown decision tree.
\begin{corollary}[Informal \Cref{cor:DT-arbitrary}]
There exists an algorithm for RPU-learning decision trees over $\mathbb{R}^d$ which in expectation:
\begin{enumerate}
    \item Makes $O\left (ds^2\log \left (\frac{s}{\varepsilon\delta} \right ) \right)$ queries
    \item Runs in time $\text{poly}(s,d,\varepsilon^{-1},\log(\delta^{-1}))$
    \item Uses $O(ds)$ memory,
\end{enumerate}
where $s$ is the size of the underlying decision tree.
\end{corollary}

Despite being vastly more expressive than axis-aligned rectangles, decision trees in $\mathbb{R}^d$ are still simplistic in the sense that they remain axis-aligned. For our final example, we study a fundamental class without any such restriction: (non-homogeneous) halfspaces in two dimensions. Recall that a halfspace in two dimensions is given by the sign of $h=\langle v, \cdot \rangle + b$ for some $v \in \mathbb{R}^2$ and $b \in \mathbb{R}$. Following a number of prior works  \cite{jamieson2011active,karbasi2012comparison, wauthier2012active,xu2017noise,hopkins2020power,hopkins2020noise,Cui2020uncertainty}, we study the learnability of halfspaces with comparison queries. Informally, given two points of the same sign, a comparison query simply asks which is further away from the separating hyperplane $h = 0$. This type of query is natural in scenarios like halfspaces where the class has an underlying ranking. One might ask a doctor, for instance, ``which patient is sicker?''. We prove that halfspaces in two dimensions have an efficiently computable lossless compression scheme of size $O(1)$ with respect to comparison queries, and thus that they are efficiently RPU-learnable in bounded memory.
\begin{corollary}[Informal \Cref{cor:halfspace}]
The class of halfspaces over $\mathbb{R}^2$ is actively RPU-learnable in only
\[
q(\varepsilon,\delta) \leq O\left(\log\left(\frac{1}{\varepsilon\delta}\right)\right)
\]
queries, $O(1)$ memory, and time $O\left( \frac{\log^2(1/(\delta\varepsilon))}{\varepsilon}\right)$.
\end{corollary}
It should be noted that KLMZ \cite{kane2017active} remark in their work that halfspaces in two dimensions should be learnable in bounded memory, but give no indication of how this might be done. This concludes the informal statement of our results, and we end the section with a roadmap of the remainder of our work which formalizes and proves the above.

In \Cref{sec:prelim}, we discuss preliminaries, including our learning model and enriched queries, and related work. In \Cref{sec:LCS}, we formally introduce lossless sample compression and prove it is is a sufficient condition for bounded memory active RPU-learnability. In \Cref{sec:examples}, we provide three fundamental classifier families with efficiently computable lossless compression schemes with respect to basic enriched queries: axis-aligned rectangles, decision trees, and halfspaces in two dimensions. We conclude in \Cref{sec:further} with a few comments on further directions.

\section{Preliminaries}\label{sec:prelim}
\subsection{Reliable and Probably Useful Learning}\label{sec:RPU}
We study a strong model of learning introduced by Rivest and Sloan \cite{rivest1988learning} called \textit{Reliable and Probably Useful Learning} (RPU-Learning). Unlike the more standard PAC setting, RPU-Learning requires that the learner never makes an error. To compensate for this stringent requirement, the learner may respond ``I don't know'', denoted ``$\bot$,'' on a small fraction of the space. In the standard, passive version of RPU-Learning, the learner has access to a sample oracle from an adversarily chosen distribution. The goal is to analyze the number of labeled samples required from this oracle to learn almost all inputs with high probability.
\begin{definition}[RPU-Learning]\label{def:RPU}
A hypothesis class $(X,H)$ is RPU-Learnable with sample complexity $n(\varepsilon,\delta)$ if there exists for all $\varepsilon,\delta >0$ a learning algorithm $A$ such that for any choice of distribution $D$ over $X$ and $h \in H$, the learner is:
\begin{enumerate}
    \item Probably useful: 
    \[
    \Pr_{S \sim D^{n(\varepsilon,\delta)}}\left [ \Pr_{x \sim D}[A(S,h(S))(x) = \bot] > \varepsilon \right ] < \delta,
    \]
    \item Reliable: 
    \[
    \forall S,x \ \text{s.t.} \ A(S,h(S))(x) \neq \bot, A(S,h(S))(x) = h(x),
    \]
\end{enumerate}
where $S,h(S)$ is shorthand for the set of labeled samples $(x,h(x))$ for $x \in S$.
\end{definition}
In other words, the learner outputs a label with high probability, and never makes a mistake. This model of learning is substantially stronger than the more standard PAC model, which need only be approximately correct. In fact, it is known that RPU-learning with only labels often has infeasibly large (or even infinite) sample complexity \cite{Kivinen,Kivinen2,hopkins2020power}. Recently, KLMZ \cite{kane2017active} proved that this barrier can be broken by allowing the learner to ask enriched questions, and gave an efficient algorithm for RPU-learning halfspaces in two dimensions (later extended in \cite{hopkins2020power} to arbitrary dimensions with suitable distributional assumptions). While the algorithms in these works give a substantial improvement over previous impossibility results, they come with a practical caveat: reaching high accuracy guarantees requires an infeasible amount of storage. In this work we show not only how to build efficient RPU-learning algorithms for a broader range of queries and hypothesis classes, but also show that this strong model remains surprisingly feasible even in scenarios where memory is severely limited.
\subsection{Active Learning}
Unfortunately, even with the addition of enriched queries, it is not in general possible to RPU (or even PAC) learn in fewer than $\text{poly}(1/\varepsilon)$ labeled samples. In cases where labels are prohibitively expensive (e.g.\ medical imagery), this creates a substantial barrier to learning even the simplest classes such as 1D-thresholds. To side-step this issue, we consider the well-studied model of \textit{active learning}. Unlike the previously discussed (passive) models, an active learner receives \textit{unlabeled samples} from the distribution and may choose whether or not to send each sample to a labeling oracle. The overall complexity of learning, called \textit{query complexity}, is then measured not by the total number of samples, but by the number of calls to the labeling oracle required to achieve the guarantees laid out in \Cref{def:RPU}. 

One might hope that the additional adaptivity allowed by active learning allows scenarios with high labeling cost to become feasible, and indeed it does in some basic or restricted scenarios, lowering the overall complexity from $\text{poly}(1/\varepsilon)$ to $\text{poly}(\log(1/\varepsilon))$ (see e.g.\ Settles' \cite{settles2009active} or Dasgupta's \cite{dasgupta2011two} surveys, or Hanneke's book \cite{hanneke2014theory}). Unfortunately, it has long been known that active learning fails to give any substantial improvement for fundamental classes such as halfspaces, even in the weaker PAC-model \cite{dasgupta2005analysis}. We continue a recent line of work showing this barrier may be broken by the same technique that permits efficient RPU-learning: asking more informative questions.
\subsection{Enriched Queries}\label{sec:queries}
Instead of having access only to a labeling oracle, the learners we discuss in this work will have the ability to ask a range of natural questions regarding the data. Enriched queries we discuss will be defined on a fixed input size we denote by $k$. A label query, for instance, is defined on a single point and thus has $k=1$. A comparison between two points would have $k=2$. While we will only consider binary labels, we will in general consider enriched queries of an arbitrary arity $r$ denoting the total number of possible answers. Binary queries like labels or comparisons have $r=2$, but some queries we consider like the ``odd-one-out'' oracle have larger arity. Finally, we will not assume that our queries have unique valid answers. Recalling again the ``odd-one-out'' query, an instance outside a rectangle may violate multiple constraints, and thus have multiple valid answers to the (informal) question ``what is wrong with (example) $x$?''

To formalize these notions, we consider each type of query to be an oracle (function) of the form:
\[
\bigo: H \times X^k \to P([r]) \setminus \{\varnothing\},
\]
where $P([r])$ denotes the powerset of $[r]=\{1,\ldots,r\}$, and $\bigo(h,T) \subseteq [r]$ denotes the set of valid responses to the query represented by $\bigo$ on $T$ under hypothesis $h$. Since in practice a user is unlikely to list all valid responses in $\bigo(h,T)$, we do not allow the learner direct access to the oracle response. Instead, when the learner queries $T$ the adversary selects a valid response from $\bigo(h,T)$ to send back.

In this work, we consider hypothesis classes $(X,H)$ endowed with a collection of oracles $\{\bigo_i\}_{i=1}^\ell$, which we collectively denote as the \textit{query set} $Q$. While each oracle in $Q$ is defined only on a certain fixed sample size, we will often wish to make every possible query associated to some larger sample $S \subset X$. In particular, given a hypothesis $h \in H$, we denote by $Q_h(S)$ the set of all possible responses to queries on $S$.\footnote{Formally, this is the product space of valid responses to each possible query on $S$.} We will generally think of the learner making queries like this in batches on a larger set $S$, and receiving some $q(S) \in Q_h(S)$ from the adversary. Additionally, it will often be useful to consider the restriction of a given $q(S)$ to queries on some subset $S' \subset S$, which we denote by $q(S)|_{S'}$. For simplicity and when clear from context, we will write just $q(S')$ as shorthand for $q(S)|_{S'}$. While there may exist many $q(S') \in Q_h(S')$ that are not equal to $q(S)|_{S'}$, we will generally be able to assume without loss of generality that our learners do not re-query anything in $S'$, which ensures the notation is well-defined. 

\subsection{Inference}\label{sec:infer}
Similar to the framework introduced in \cite{kane2017active}, we will often wish to analyze what information is \textit{inferred} by a certain query response $q(S) \in Q_h(S)$. Let $(X,H)$ be a hypothesis class with associated query set $Q=\{\bigo_i\}$. Given a sample $S \subset X$ and query response $q(S) \in Q_h(S)$, denote by $H|_{q(S)}$ the set of hypotheses consistent with $q(S)$. For any oracle $\bigo_i$ and appropriately sized subset $T \subset X$, we say that $q(S)$ infers $\alpha \in \bigo_i(h,T)$ if $\alpha$ is a valid query response with respect to every consistent hypothesis: 
\begin{align}\label{eq:infer}
\forall h' \in H|_{q(S)}, \alpha \in \bigo_i(h',T).
\end{align}
It is worth noting that since $\bigo_i(h',T)$ may contain multiple valid responses, it is possible that $q(S)$ may infer several of them. As in the previous section, it will often be useful to consider this process in batches. In particular, given a query set $Q$, $q(S) \in Q_h(S)$, and $q(S') \in Q_h(S')$, we may wish to know when $q(S)$ infers that $q(S')$ is a valid response in $Q_h(S')$ for all $h \in H|_{q(S)}$. In such cases, we say $q(S)$ infers $q(S')$ and write:
\[
q(S) \rightarrow q(S').
\]
As in previous works \cite{kane2017active,kane2018generalized,HarPeled2020ActiveLA,hopkins2020noise,hopkins2020power}, we pay particular attention to the case where $\bigo_i=L$ is the labeling oracle (notation we will use throughout). We introduce two important concepts for this special case. Given a sample $S$ and query response $q(S)$, it will be useful to analyze the set of points in $X$ whose labels are inferred by $q(S)$. We denote this set by $I(q(S))$. Similarly, it will be useful to analyze how much of $X$ $I(q(S))$ covers (with respect to the distribution over $X$), which we call the \emph{coverage} of $q(S)$ and denote by $\text{Cov}(q(S))$. Finally, when dealing with labels we may wish to restrict the scope of our inference for the sake of computational efficiency. In such scenarios, we will define an \emph{inference rule} $R$, which for each query response $q(S)$ determines some subset $S \subseteq I_R(q(S)) \subseteq I(q(S))$. We let $\text{Cov}_R$ denote the coverage with respect the rule $R$, and $T_{I_R}(n)$ denote the inference time under rule $R$---that is the worst-case time across $x \in X$, $S \subset X$, and $q(S) \in Q_h(S)$ to determine whether $x \in I_R(q(S))$. Finally, we call the rule $R$ efficiently computable if it is polynomial in $n$. When $R$ is trivial (i.e. $\forall q(S): I_R(q(S))=I(q(S))$), we drop it from all notation.

\subsection{Bounded Memory}

The main focus of our work lies in understanding not only when active learning can achieve exponential improvement over passive learning, but when this can be done by a learner with limited memory. Previous works studying active learning with enriched queries mostly focus on the pool-based model, where the learner receives a large batch of unlabeled samples rather than access to a sampling oracle. In this case, the implicit assumption is that the learner may query any subset of samples from the pool, but this requires the learner to use a large amount of storage.

Adapting definitions from the passive learning literature \cite{h-sela-88,floyd1989space,ameur1993trial}, we define a new, more realistic model for active learning with enriched queries in which the learner may only store some finite number of points from a stream of samples. At any given step, we restrict the learner to querying only points it currently remembers. More formally, we consider learners equipped with two tapes, the \textit{query} and \textit{work} tapes, and two counters, the \textit{sample} and \textit{query} counters. The query tape stores points in the instance space $X$ that the learner has saved and may wish to query. The work tape, on the other hand, stores bits which provide any extra information about these points needed for computation---typically this entails query responses, but we will see cases where other types of information are useful as well. The sample and query counters, true to name, track the total number of unlabeled samples drawn and queries made by the algorithm at any given step.

We note that the complexity of the query tape is measured in the number of points stored there at any given time, rather than in the total number of bits required to represent it. This matches early works on bounded-memory learning in the passive regime \cite{h-sela-88,floyd1989space,ameur1993trial} and is necessary due to the fact that we are interested mainly in working over infinite instance spaces like the reals (where representing even a single point may take infinite bits). It is also worth noting that this avoids the fact that different representations of data may have different bit complexities---given a certain representation of the data (say with finite bit-complexity), it is easy to convert our memory bounds if desired to a model counting only bits. 

We now discuss our model in greater depth. Given the query and work tapes, the learner may choose at each step from the following options:
\begin{enumerate}
    \item Sample a point and add it to the query tape.
    \item Remove a point from the query tape.
    \item Query any subset of points on the query tape, writing the results on the work tape.
    \item Write or remove a bit from the work tape.
    % \item Read the sample or query counter.\footnote{More formally, we mean by this that the action of the algorithm at any step is controlled by a transition function that may depend on the values of the sample and query counters.}
\end{enumerate}
Further, as in previous work \cite{ameur1993trial}, we allow the action taken by the algorithm at any step to depend on the contents of the query and sample counters. This can be formalized in one of two ways. The first, considered implicitly by Ameur et al.\ \cite{ameur1993trial}, is to think of the algorithm as governed by a non-uniform transition function that may depend on the entire content of the sample and query counters. We will generally take this view throughout the paper since it is simpler, but if one wishes to use a uniform model of computation, another method is to allow the algorithm to run ``simple'' randomized procedures that only take about loglog space in the size of the counter. Since in general our algorithms use at most $n=\text{poly}(\varepsilon^{-1},\log(1/\delta))$ samples, the latter view essentially allocates a special block of $O(\log\log(1/\varepsilon) + \log\log(1/\delta))$ memory\footnote{In a bit more detail, we can use approximate counting to probabilistically estimate the counter up to a small constant factor with probability at least $1-\delta$ in space $O(\log(\log(n)) + \log(\log(1/\delta)))$ \cite{nelson2020optimal}, which explains the discrepancy in dependence on $\delta$ in these two equations.} to deal with the counter. We note that because we allow this procedure to be randomized, this version of the model requires expected rather than worst-case bounds on query and computational complexity.

With this in mind, we say that a hypothesis class is \textit{RPU-learnable with bounded memory} if there exists an RPU-learner for the class whose query and work tapes never exceed some constant $M(X,H)$ length independent of the learning parameters $\varepsilon$ and $\delta$. At a finer grain level, we say that such a class is learnable with memory $M(X,H)$. We emphasize that as in \cite{ameur1993trial}, we do not measure the memory usage of the counter. Indeed, it is not hard to see that some sort of counter or memory scaling with $\varepsilon$ and $\delta$ is necessary to give a stopping condition for the learner. Finally, we note that while previous techniques such as the inference dimension algorithm of \cite{kane2017active} can certainly be modified to fit into the above framework, they do not result in bounded memory learners, requiring storage that scales with $\varepsilon$ and $\delta$ in both the query and work tapes.
\subsection{Related Work}
\subsubsection{Bounded Memory Learning}
Bounded memory learning in the sense we consider was first introduced in an early work by Haussler \cite{h-sela-88}, who showed the existence of passive PAC-learners for restricted classes of decision trees and basic functions on $\mathbb{R}$ such as finite unions of intervals with memory independent of the accuracy parameters $\varepsilon$ and $\delta$.  Floyd \cite{floyd1989space}, and later Ameur, Fischer, Hoffgen, and Meyer auf der Heide \cite{ameur1993trial}, extended Haussler's work to a general theory of bounded memory passive learning including features we consider such as a sample counter. The latter in particular give necessary and sufficient conditions based upon a compression scheme stronger than standard sample compression, but weak enough to cover important classifiers such as halfspaces. In the decades since, many variants and relaxations of bounded memory learning (e.g. memory scaling with $\varepsilon,\delta$, learning a finite stream, storing only bits, time-space tradeoffs, etc.) have seen a substantial amount of study \cite{shamir2014fundamental,steinhardt2016memory,raz2017time,moshkovitz2017mixing,moshkovitz2017general,moshkovitz2018entropy,beame2018time,raz2018fast,garg2018extractor,sharan2019memory,dagan2019space,gonen2020towards,assadi2020near}. To our knowledge, however, the subject has seen little to no work within the active learning literature, though a few do consider query settings beyond just labels including statistical \cite{steinhardt2016memory,gonen2020towards} and equivalence queries \cite{ameur1995space}.
\subsubsection{Active Learning with Enriched Queries}
Active learning with enriched queries has become an increasingly popular alternative to the standard model in cases like halfspaces where strong lower bounds prevent adaptivity from providing a significant advantage over standard passive learning. While most prior works in the area consider specific examples of enriched queries such as comparisons \cite{jamieson2011active,karbasi2012comparison, wauthier2012active,xu2017noise,hopkins2020power,hopkins2020noise,Cui2020uncertainty}, cluster-queries \cite{ashtiani2016clustering, vikram2016interactive}, mistake queries \cite{balcan2012robust}, and separation queries \cite{HarPeled2020ActiveLA}, our work is more closely related to the general paradigm for enriched query active learning introduced by KLMZ \cite{kane2017active}. In their work, KLMZ introduce \textit{inference dimension}, a combinatorial parameter that exactly characterizes when a concept class is actively learnable in $O(\log(1/\varepsilon))$ rather than $O(1/\varepsilon)$ queries. Lossless sample compression can be seen as a strengthening of inference dimension (albeit extended to a richer regime of queries than that considered in \cite{kane2017active}) which implies both $O(\log(1/\varepsilon))$ query complexity and bounded memory. This partially answers an open question posed by KLMZ, who asked whether any class with finite inference dimension has a bounded memory learner.

\section{Lossless Sample Compression}\label{sec:LCS}

In this section, we prove that lossless sample compression is sufficient for query-optimal, bounded memory RPU-learning. While we introduced the concept of lossless compression in  \Cref{sec:intro}, we are now in position to state the full formal definition.
\begin{definition}[Lossless Compression Schemes]\label{def:LCS}
We say a hypothesis class $(X,H)$ has a lossless compression scheme (LCS) of size $k$ with respect to a query set $Q$ and inference rule $R$ if for all classifiers $h \in H$, monochromatic subsets\footnote{A subset $S$ is monochromatic with respect to a classifier $h$ if it consists entirely of one label.} $S \subset X$, and any $q(S) \in Q_h(S)$ there exists $W=W(q(S)) \subseteq S$ of size at most $k$ such that:
\[
I_R(q(S)) = I_R(q(S)|_W).
\]
Let $T_C(n)$ denote the worst-case time required to determine such a subset. We call the LCS efficiently computable if $T_C(n)$ it is polynomial in $n$. If no inference rule $R$ is stated, it is assumed to be the trivial rule.
\end{definition}

It is worth briefly noting the intuition behind requiring compression only for monochromatic sets. The general idea is that in RPU-learning it is sufficient to be able to learn $X$ restricted to the set of positive and negative points. While this strategy fails in the weaker PAC-model (the learner could output the all $1$'s function for positive points for instance), the fact that RPU-learners cannot make mistakes circumvents this issue.

We will begin by proving that lossless compression implies sample-efficient passive RPU-learning, and then show how this can be leveraged to give query-efficient and bounded-memory active learning. In fact, if all one is interested in is the former, lossless sample compression is needlessly strong. As a result, we start by analyzing a strictly weaker variant that lies in between standard and lossless sample compression, and bears close similarities to KLMZ's \cite{kane2017active} theory of inference dimension.
\begin{definition}[Perfect Compression Schemes (PCS)]
We say a hypothesis class $(X,H)$ with corresponding query set $Q$ has a perfect compression scheme (PCS) of size $k$ with respect to inference rule $R$ if for all subsets $S \subset X$, $h \in H$, and any $q(S) \in Q_h(S)$ there exists $T=T(Q(S)) \subseteq S$ of size at most $k$ such that:
\[
S \subseteq I_R(q(S)|_T).
\]
\end{definition}
Thus in a perfect compression scheme, one need only recover the labels of the original sample, while lossless compression requires that any labels \emph{inferred} by queries on the original sample are also preserved. Rather than directly analyzing the effect of PCS on active learning, we first prove as an intermediate theorem that such schemes are sufficient for near-optimal passive RPU-learnability. Combining this fact with the basic boosting procedure of \cite{kane2017active} gives query-optimal (but not bounded memory) active RPU-learnability. The argument for the passive case closely follows the seminal work of Floyd and Warmuth \cite[Theorem 5]{floyd1995sample} on learning via sample compression.
\begin{theorem}\label{passive-perfect}
Let $(X,H)$ have a perfect compression scheme of size $k$ with respect to inference rule $R$. Then the sample complexity of passively RPU-learning $(X,H)$ is at most
\[
n(\varepsilon,\delta) \leq O\left(\frac{k\log(1/\varepsilon) + \log(1/\delta)}{\varepsilon} \right).
\]
\end{theorem}
\begin{proof}
Let $h\in H$ be an arbitrary classifier. Our goal is to upper bound by $\delta$ the probability that for a random sample $S$, $|S|=m$, there exists $q(S) \in Q_h(S)$ such that $\text{Cov}_R(q(S)) \leq 1-\varepsilon$. Notice that it is equivalent to prove that for some fixed worst-case choice of $q(S)$ (minimizing coverage across each sample):
\[
\Pr_S[\text{Cov}_R(q(S)) \leq 1-\varepsilon] \leq \delta.
\]
Given this formulation, the argument proceeds along the lines of \cite[Theorem 5]{floyd1995sample}. It is enough to upper bound by $\delta$ the probability across samples $S$ of size $m$ that there exists $T \subset S$ of size at most $k$ such that:
\begin{enumerate}
    \item $S \subseteq I_R(q(S)|_T)$
    \item $\text{Cov}_R(q(S)|_T) < 1-\varepsilon$
\end{enumerate}
Since the PCS implies that every set $S$ has a subset $T$ satisfying the first condition, if both hold only with some small probability then the following basic algorithm $A$ suffices to RPU-learn with sample complexity $m$: on input $x \in X, A(S)(x)$ checks whether all $h \in H$ consistent with $q(S)$ satisfy $L(h,x)=z$ for some $z\in \{0,1\}$.\footnote{Note that this process is class-dependent.} If this property holds, the algorithm outputs $z$. Otherwise, the algorithm outputs ``$\bot$.'' Since with probability at least $1-\delta$ the coverage of $q(S)$ is at least $1-\varepsilon$, this algorithm will label at least a $1-\varepsilon$ fraction of the space while never making a mistake. 

Proving this statement essentially boils down to a double sampling argument. The idea is to union bound over sets of indices in $[m]$ of size up to $k$, noting that in each case the remaining $m-k$ points can be treated independently. In greater detail, for $I \subset [m]$, denote by $B_I$ the set of samples $S$ which are inferred by the subsample with indices given by $I$, that is:
\[
B_I = \{ S=\{s_1,\ldots,s_m\}: S \subseteq I_R(q(S)|_{\{s_i\}_{i \in I}}) \}.
\]
On the other hand, let $U_I$ denote samples where the coverage of the subsample given by $I$ is worse than $1-\varepsilon$:
\[
U_I = \left\{ S=\{s_1,\ldots,s_m\}: \text{Cov}_R\left(q(S)|_{\{s_i\}_{i \in I}}\right) < 1-\varepsilon \right\}.
\]
Notice that the intersection of $B_I$ and $U_I$ is exactly what we are trying to avoid. By a union bound, the probability that we draw a sample such that there exists a subset $T$ satisfying 1 and 2 is then at most:
\[
\sum\limits_{I \subset [m], |I| \leq k} \Pr_{S}[S \in B_I \cap U_I].
\]
It is left to bound the probability for fixed $I \subset [m]$ of drawing a sample in $B_I \cap U_I$. Since we are sampling i.i.d, we can think of independently sampling the coordinates in and outside of $I$. If the samples given by $I$ have coverage at least $1-\varepsilon$, we are done. Otherwise, the probability that we draw remaining samples that are in the coverage is at most $(1-\varepsilon)^{m-|I|}$. Thus we get that the event is bounded by:
\[
(1-\varepsilon)^{m-k}\sum\limits_{i=1}^k {m \choose i},
\]
which is at most $\delta$ for $m = O\left(\frac{k\log(1/\varepsilon) + \log(1/\delta)}{\varepsilon} \right)$
\end{proof}
\Cref{passive-perfect} implies that hypothesis classes with perfect compression schemes may be passively learned with nearly equivalent asymptotic sample complexity to the much weaker PAC-model (off only by a $\log(1/\varepsilon)$ factor recently removed by Hanneke \cite{hanneke2016optimal} in the latter case). Further, since perfect sample compression is preserved over subsets of the instance space (a PCS for $X$ is also a PCS when restricted to $S \subset X$), a bound on passive RPU learning immediately implies efficient active learning via a modification to the basic boosting strategy for finite instance spaces of \cite[Theorem 3.2]{kane2017active}.
\begin{theorem}[Active Learning]\label{cor:AL}
Let $(X,H)$ have a perfect compression scheme of size $k$ with respect to query set $Q$ and inference rule $R$. Then $(X,H)$ is actively RPU-learnable in only
\[
q(\varepsilon,\delta) \leq O\left( b(6k)\log\left(\frac{1}{\varepsilon\delta}\right)\right)
\]
queries, and
\[
T(\varepsilon,\delta) \leq O\left(\left(t(6k)+ \frac{kT_{I_R}(60k\log(1/(\varepsilon\delta)))\log(k/(\varepsilon\delta))}{\varepsilon}\right)\log\left(\frac{1}{\varepsilon\delta}\right)\right)
\]
time where $b(n)$ is the worst-case number of queries needed to infer some valid $q(S) \in Q_h(S)$ across all $h \in H$ and $|S|=n$, and $t(n)$ is the worst-case time.
\end{theorem}
\begin{proof}
For notational convenience, let $X'$ be a copy of $X$ we use to track un-inferred points throughout our algorithm. Consider the following strategy:
\begin{enumerate}
    \item Sample $S$ of size $6k$ from $X'$ (note this may require taking many samples from $X$ in later rounds)
    \item Infer some $q(S) \in Q_h(S)$ via queries on $S$
    \item Restrict $X'$ to points whose labels are not inferred by $q(S)$, (i.e. remove any $x \in X'$ s.t.\ $q(S) \rightarrow L(h,x)$)\footnote{In reality this is done by rejection sampling in step 1. When a point is drawn, we check whether it is inferred by queries on any previous sample.}
    \item Repeat $O\left(\log\left(\frac{1}{\varepsilon\delta}\right)\right)$ times, or until $n=O\left(\frac{k\log(k/(\varepsilon\delta))\log(1/(\varepsilon\delta))}{\varepsilon}\right)$ total points have been drawn from $X$.
\end{enumerate}
Notice that for either of these stopping conditions, one of two statements must hold:
\begin{enumerate}
    \item The algorithm has performed at least $O\left(\log\left(\frac{1}{\varepsilon\delta}\right)\right)$ rounds.
    \item The algorithm has drawn $O\left(\frac{\log\left(\frac{n}{\delta}\right)}{\varepsilon}\right)$ inferred samples in a row.
\end{enumerate}
We argue both of these conditions imply that the coverage of the learner is at least $1-\varepsilon$ with probability at least $1-\delta$. For the first condition, notice that \Cref{passive-perfect} implies the coverage of $q(S)$ on a random sample of size $6k$ is at least $1/2$ with probability at least $1/2$. Call a round `good' if it has coverage at least $1/2$. It is sufficient to prove we have at least $\log(1/\varepsilon)$ good rounds with probability at least $1-\delta$. Since each round can be thought of as an independent process, this follows easily from a Chernoff bound. For the latter condition, notice that the probability $O\left(\frac{\log\left(\frac{n}{\delta}\right)}{\varepsilon}\right)$ inferred samples appear in a row (at some fixed point) when the coverage is less than $1-\varepsilon$ is at most $O(\delta/n)$. Union bounding over samples implies the algorithm has the desired coverage guarantees.

Finally, we compute the query and computational complexity. The former follows immediately from noting that we make at most $b(6k)$ queries in each round, and run at most $O\left(\log\left(\frac{1}{\varepsilon\delta}\right)\right)$ rounds. The latter bound stems from noting that each round takes at most $t(6k)$ time to compute queries, and each \emph{sample} requires at most $T_{I_R}(ck\log(1/\varepsilon\delta))$ time for inference for some $60>c>0$.
\end{proof}
A similar result may also be proved by defining a suitable extension of inference dimension to our generalized queries and noting that a perfect compression scheme implies finite inference dimension. It is also worth noting that in most cases of interest the above algorithm will also be computationally efficient, as the cost of compression, querying, rejection sampling, and inference tend to be fairly small (and are in all examples we consider). Finally, we show how \Cref{cor:AL} can be modified with the addition of an efficiently computable lossless compression scheme to give query-optimal, computationally efficient, and bounded memory RPU-learning.

\begin{theorem}[Bounded Memory Active Learning]\label{cor:bounded}
Let $(X,H)$ have an LCS of size $k$ with respect to query set $Q$ and inference rule $R$, and let $B(n)$ denote the maximum number of bits required to express any $q(S) \in Q_h(S)$ for $S \subset X$ of size $n$ and $h \in H$. Then $(X,H)$ is actively RPU-learnable in only:
\[
q(\varepsilon,\delta) \leq O\left(b(6k)\log\left(\frac{1}{\varepsilon\delta}\right)\right)
\]
queries,
\[
M(X,H) \leq O(B(7k))
\]
memory, and
\[
T(\varepsilon,\delta) \leq O\left(\left(t(7k)+T_C(7k) + \frac{kT_{I_R}(k)\log(k/(\varepsilon\delta))}{\varepsilon}\right)\log\left(\frac{1}{\varepsilon\delta}\right)\right)
\]
time.
\end{theorem}
\begin{proof}
We follow the overall strategy laid out in \Cref{cor:AL} with two main modifications. First, instead of storing subsamples and queries from all previous rounds to use for rejection sampling, each sub-sample will be merged in between rounds using the guarantees of lossless sample compression. Without this change, the strategy of \Cref{cor:AL} would require $O\left (B(6k)\log\left ( \frac{1}{\varepsilon\delta}\right) \right )$ memory from the buildup across rounds. Second, we will learn the set of positive and negative points separately since the LCS only guarantees compression for monochromatic subsamples.

More formally, assume for the moment that we draw samples from the distribution restricted to positive (or negative) labels. Following the strategy of \Cref{cor:AL}, let the sample of size $6k$ drawn at step $i$ be denoted by $S_i$, the points stored in the query-tape at the start of step $i$ by $C_{i-1}$, and their corresponding query response $q(C_{i-1})$. 
%For simplicity, assume the learner starts by making all queries on $S_i$ to receive $q(S_i)$, and then the remaining queries on $S_i \cup C_{i-1}$ to receive $q(S_i \cup C_{i-1})$. Since the learner already had access to $q(C_{i-1})$ from the previous round, this means we can write $q(S_i \cup C_{i-1})|_{S_i}$ and $q(S_i \cup C_{i-1})|_{C_{i-1}}$ as just $q(S_i)$ and $q(C_{i-1})$ respectively. 
The existence of an LCS of size $k$ implies for any $h \in H$ and $q(S_i \cup C_{i-1}) \in Q_h(S_i \cup C_{i-1})$, there exists a subset $C_{i} \subset S_i \cup C_{i-1}$ of size at most $k$ such that: 
\[
I_R(q(C_i)) = I_R(q(S_i \cup C_{i-1})),
\]
where we recall $q(C_i)$ is shorthand for $q(S_i \cup C_{i-1})|_{C_i})$. Further, since restrictions have strictly less information (that is for all $S,S' \subseteq X$ and $q(S \cup S') \in Q_h(S \cup S')$, $I_R(q(S)) \cup I_R(q(S')) \subseteq I_R(q(S \cup S'))$, we may write:
\begin{align*}
I_R(q(S_i)) \cup I_R(q(C_{i-1})) &\subseteq I_R(q(S_i \cup C_{i-1})) = I_R(q(C_i)).
\end{align*}
By induction on the step $i$, $q(C_{i-1})$ infers the label of every point in $I(q(S_j))$ for $j < i$, and therefore:
\[
\bigcup_{j=1}^i I_R(q(S_j)) \subseteq
I_R(q(C_i)).
\]
Since the coverage guarantees of \Cref{cor:AL} at a given step $i$ rely only on this left-hand union, following the same strategy (plus compression) still results in at least $1-\varepsilon$ coverage with probability at least $1-\delta$ in only $O(\log(\frac{1}{\varepsilon\delta}))$ rounds. Further, since we merge our stored information every round into $C_i$ and $q(C_i)$, we never exceed storage of $7k$ points and $O(B(7k))$ bits. The analysis of query and computational complexity follow the same as in \Cref{cor:AL} with the addition of compression time. 

Finally, we argue that we may learn the general distribution by separately applying the above to the set of positive and negative samples. Namely, at step $i$ we sample from the remaining un-inferred points until we receive either $6k$ positive or $6k$ negative samples, and apply the standard algorithm on whichever is reached first. Assume at step $i$ that the measure of remaining positive points is at least that of the remaining negative (the opposite case will follow similarly). Then the probability that we draw $6k$ positive points before $6k$ negative points is at least $1/2$. Recall that such a positive sample has coverage at least $1/2$ over the remaining positive points with probability at least $1/2$. Moreover, since the positive samples make up at least $1/2$ of the space, the coverage of each round over the \emph{entire} distribution is at least $1/4$ with probability at least $1/4$ (factoring in the probability we draw the majority sign). Achieving the same guarantees as above then simply requires running a small constant times as many rounds. Thus there is no asymptotic change to any of the complexity measures and we have the desired result.
\end{proof}
We note that in all examples considered in this paper, $b(n)$, $T_{I_R}(n)$, and $T_C(n)$ are at worst quadratic, resulting in computationally efficient algorithms. Finally, it is worth briefly discussing the fact that it is possible to remove the query counter in the proof of \Cref{cor:bounded} if one is willing to measure \emph{expected} rather than \emph{worst-case} query-complexity. This mainly involves running the same algorithm using only the sample cutoff and analyzing the probability that the algorithm makes a large number of queries.

\section{Three Classes with Finite LCS}\label{sec:examples}
In this section we cover three examples of fundamental classes with small, efficiently computable lossless compression schemes: axis-aligned rectangles, decision trees, and halfspaces in two dimensions. In each of these cases standard lower bounds show that without additional queries, active learning provides no substantial benefit over its passive counterpart \cite{dasgupta2005analysis}. Furthermore, these classes cannot be actively RPU-learned at all (each requires an infinite number of queries) \cite{hopkins2020noise}. Thus we see that the introduction of natural enriched queries brings learning from an infeasible or even impossible state to one that is highly efficient, even on a very low memory device.

\subsection{Axis-Aligned Rectangles}
Despite being one of the most basic classes in machine learning, axis-aligned rectangles provide an excellent example of the failure of both standard active and RPU-learning. In this section we show that the introduction of a natural enriched query, the \textit{odd-one-out} query, completely removes this barrier, providing a query-optimal, computationally efficient RPU-learner with bounded memory. Recall that axis-aligned rectangles over $\mathbb{R}^d$ are the indicator functions corresponding to products of intervals over $\mathbb{R}$:
\[
R = [a_1,b_1] \times \ldots \times [a_d,b_d]
\]
for $a_i \leq b_i$. Additionally we allow $a_i$ to be $-\infty$, and $b_i$ to be $\infty$ (though in this case the interval should be open). In other words, thinking of each coordinate in $\mathbb{R}^d$ as a feature, axis-aligned rectangles capture scenarios where each individual feature has an acceptable range, and the overall object is acceptable if and only if every feature is acceptable. For instance, one might measure a certain dish by flavor profile, including features such as saltiness, sourness, etc. If the dish is too salty, or too sour, the diner is unlikely to like it independent of its other features. In this context, the ``odd-one-out'' query asks the diner, given a negative example, to pick a specific feature they did not like, and further to specify whether the feature was either missing, or too present. Perhaps the dish was too sour, or needed more umami. 

More formally, the odd-one-out oracle $\bigo_{\text{odd}}: H \times X \to P\left (\left ([d] \times \{0,1\}\right ) \cup \{*\}\right)$ on input $h=[a_1,b_1] \times \ldots \times [a_d,b_d]$ and $x \in \mathbb{R}^d$ outputs $\{*\}$ if $L(h,x)=1$, and otherwise outputs the set of pairs $(i,1)$ such that $x_i>b_i$ and $(i,0)$ such that $x_i < a_i$.
% an odd-one-out oracle if for every $h \in H$ and negative example $x \in X$, $\bigo(h,x)$ contains $(i,b)$ satisfying:
% \[
% \begin{cases}
% x_i > b_i & \text{if } b = 1\\
% x_i < a_i & \text{if } b = 0,
% \end{cases}
% \]
% and for every positive example $x$, $\bigo(h,x)=\{*\}$.
\begin{proposition}\label{prop:rect}
The class of axis-aligned rectangles over $\mathbb{R}^d$ has a lossless compression scheme of size at most $2d$ with respect to $\bigo_{\text{odd}}$.
\end{proposition}
\begin{proof}
Recall that lossless sample compression separately examines the set of positively and negatively labeled points. We first analyze the positive samples---those which lie inside the axis-aligned rectangle. In this case, notice that we may simply select a subsample $T$ of size at most $2d$ which contains points with the maximum and minimal value at every coordinate. Since it is clear that $T$ infers all points inside the rectangle  $R(T)$ it spans, it is enough to argue that $S$ cannot infer any point outside $R(T)$. This follows from the fact that both $R(T)$ and the all $1$'s function are consistent with queries on $S$.

For the case of a negatively labeled sample, notice that if any two $x,x' \in X$ have the same ``odd-one-out'' response $(b,i)$, one must infer the response of the other. Informally, this is simply because the odd-one-out query measures whether some feature is too large or too small. If two points are too large in some coordinate, say, then the point with the smaller feature infers this response in the other point. More formally, assume without loss of generality that $x_i \leq x'_i$. Then if the query response is $(0,i)$, $q(S)|_{x'} \rightarrow q(S)|_{x}$. Likewise, if the response is $(1,i)$, then $q(S)|_{x} \rightarrow q(S)|_{x'}$. Thus for each of the $2d$ query types we need only a single point to recover all query information for the original sample. Since no information is lost, this compressed set clearly satisfies the conditions of the desired LCS.
\end{proof}
As an immediate corollary, we get that axis-aligned rectangles are actively RPU-learnable with near-optimal query complexity and bounded memory. Since the compression scheme and inference rule are efficiently computable, the algorithm is additionally computationally efficient.
\begin{corollary}\label{cor:rect}
The class of axis-aligned rectangles in $\mathbb{R}^d$ is RPU-learnable in only
\[
q(\varepsilon,\delta) = O\left (d\log \left (\frac{1}{\varepsilon\delta} \right ) \right)
\]
queries, $O(d)$ memory, and time $O\left(\frac{d^2\log(d/(\varepsilon\delta))\log(1/(\varepsilon\delta))}{\varepsilon}\right)$ when the learner has access to $\bigo_{\text{odd}}$.
\end{corollary}
\begin{proof}
The first two statements follow immediately from \Cref{cor:bounded} and noting that $b(n) \leq O(n)$. The runtime guarantee is slightly more subtle, and follows from noting that compression time $T_C(n) \leq O(dn)$, and that with a suitable representation of the compressed set, inference for a point $x\in\mathbb{R}^d$ only requires checking the value of each coordinate (in essence, we may act as though $T_{I}(d) \leq O(d)$). 
\end{proof}
Fixing parameters other than $\varepsilon$, we note that the optimal query complexity for axis-aligned rectangles (and all following examples) is $\Omega(\log(1/\varepsilon))$. This follows from standard arguments \cite{kulkarni1993active}, and can be seen simply by noting that there exists a distribution with at least $\Omega(1/\varepsilon)$ $\varepsilon$-pairwise-separated concepts---the bound comes from noting that each query only provides $O(1)$ bits of information. 

\subsection{Decision Trees}
While axis-aligned rectangles provide a fundamental example of a classifier for which enriched queries break the standard barriers of active RPU-learning and bounded memory, they are too simple a class in practice to model many situations of interest. In this section we consider a broad generalization of axis-aligned rectangles that is not only studied extensively in the learning literature \cite{ehrenfeucht1989learning,kushilevitz1993learning,o2007learning,kalai2008decision,blanc2020universal}, but also used frequently in practice \cite{quinlan1986induction,umanol1994fuzzy,hssina2014comparative,singh2014comparative}: decision trees. In this setting we consider a natural enriched query which roughly falls into a paradigm known as \textit{same-cluster} queries \cite{ashtiani2016clustering}, which determine whether a given set of points lie in a ``cluster'' of some sort. Variants of this query have seen substantial study in the past few years in both clustering and learning \cite{verroios2017waldo,mazumdar2017clustering,ailon2018approximate,firmani2018robust,saha2019correlation} after their recent formal introduction by Ashtiani, Kushagra, and Ben-David \cite{ashtiani2016clustering}.

% More formally, we study the class of decision trees over $\mathbb{R}^d$ with at most $s$ nodes, and prove they have an LCS of size O(ds). We further show how this result may be bootstrapped into an attribute efficient algorithm for decision trees of arbitrary (unknown) size whose expected query-efficiency and worst-case memory usage scale with size. 

More formally, recall that a decision tree over $\mathbb{R}^d$ is a binary tree where each node in the tree corresponds to an inequality:
\[
x_i \overset{?}{\geq} b \ \text{or} \ x_i \overset{?}{\leq} b,
\]
measuring the $i$-th feature (coordinate) of any $x \in \mathbb{R}^d$. Each leaf in the tree is assigned a label, and the label of any $x \in X$ is uniquely determined by the leaf resulting from following the decision tree from root to leaf, always taking the path determined by the inequality at each node. We measure the size of a decision tree by counting its leaves. We introduce a strong enriched query for decision trees in the same-cluster paradigm called the \textit{same-leaf} oracle $\bigo_{\text{leaf}}$. In particular, given a decision tree $h$ and two points $x,x'$, the same-leaf query on $x$ and $x'$ simply determines whether $x$ and $x'$ lie in the same leaf of $h$. It is worth noting that the same-leaf oracle can be seen as a strengthening of Dasgupta, Dey, Roberts, and Sabato's \cite{dasgupta2018learning} ``discriminative features,'' a method of dividing a decision tree into clusters, each of which should have some discriminating feature.  Same-leaf queries take this idea to its logical extreme where each leaf should be thought of as a separate cluster. 

We will show that same-leaf queries have a small and efficiently computable lossless compression scheme with respect to an efficient inference rule $R_{\text{rect}}$, and therefore have query-optimal and computationally efficient bounded-memory RPU-learners. The inference rule $R_{\text{rect}}$ is a simple restriction where $x \in I_{R_{\text{rect}}}(q(S))$ if and only if $x$ lies inside a rectangle spanned by a mono-leaf subset $T \subseteq S$, that is a subset $T$ for which $\forall x,x' \in T, \bigo_{\text{leaf}}(h,x,x')=1$. In other words, we infer information independently for each leaf.
\begin{proposition}\label{prop:same-leaf}
The class of size $s$ decision trees over $\mathbb{R}^d$ has an LCS of size at most $2ds$ with respect to $\bigo_{\text{leaf}}$ and $R_{\text{rect}}$.
\end{proposition}
\begin{proof}
Notice that any leaf of a decision tree corresponds to an axis-aligned rectangle in $\mathbb{R}^d$. Further, for any $S \subset \mathbb{R}^d$, $q(S)$ allows us to partition $S$ into subsamples sharing the same leaf. Fix one such subsample and denote it by $S_L$. By the same argument as \Cref{prop:rect}, there exists $T \subseteq S_L$ of size at most $2d$ such that $S_L$ lies entirely within the rectangle spanned by $T$. Since these are exactly the points inferred by $S$ in that leaf under inference rule $R_{\text{rect}}$, repeating this process over all $s$ leaves gives the desired LCS.
\end{proof}
\begin{corollary}\label{cor:DT-fixed}
The class of size $s$ decision trees is RPU-learnable in only
\[
q(\varepsilon,\delta) = O\left (ds^2\log \left (\frac{1}{\varepsilon\delta} \right ) \right)
\]
queries, $O(ds)$ memory, and time $O\left(\frac{d^2s^2\log(ds/(\varepsilon\delta))\log(1/(\varepsilon\delta))}{\varepsilon}\right)$ when the learner has access to $\bigo_{\text{leaf}}$.
\end{corollary}
\begin{proof}
The proof follows similarly to \Cref{cor:rect}. It is not hard to see that $T_C(n)$ is at worst quadratic. Determining queries on a set of size $n$ reduces to grouping points into $s$ buckets (each standing for some leaf), which can be done in $b(n) \leq O(ns)$ queries. While $T_{I_{R_{\text{rect}}}}(n)$ would technically require finding the rectangles implied by the given sample and checking $x$ against each one, the former can be thought of pre-processing since it need only be done once per round. With this process removed, checking $x$ takes only $O(d)$ time, which gives the desired result.
\end{proof}

While it is not possible to learn decision trees of arbitrary size in the standard learning models \cite{hancock1996lower}, one might hope that \Cref{cor:DT-fixed} can be used to build a learner for this class with nice guarantees. Indeed this is possible if we weaken our memory and computational complexity measures to be expected rather than worst-case. We will show that in this regime, \Cref{cor:DT-fixed} can be bootstrapped to build an \textit{attribute-efficient} algorithm for RPU-learning decision trees, that is one in which the expected memory, query, and computational complexity scale with the (unknown) size of the underlying decision tree.

\begin{corollary}\label{cor:DT-arbitrary}
There exists an algorithm for RPU-learning decision trees over $\mathbb{R}^d$ which in expectation:
\begin{enumerate}
    \item Makes $O\left (ds^2\log \left (\frac{s}{\varepsilon\delta} \right ) \right)$ queries
    \item Runs in time $\text{poly}(s,d,\varepsilon^{-1},\log(\delta^{-1}))$
    \item Uses $O(ds)$ memory,
\end{enumerate}
where $s$ is the size of the underlying decision tree.
\end{corollary}
\begin{proof}
For simplicity, we use \Cref{cor:DT-fixed} as a blackbox for an increasing ``guess'' $s'$ on the size of the decision tree. After each application of \Cref{cor:DT-fixed}, we test its coverage empirically. If too many points go un-inferred, we double our guess for $s'$ and continue the process. More formally, we start by setting our guess $s'$ to $2$. The learner then applies \Cref{cor:DT-fixed} with accuracy parameters $\delta'=(\delta/s')$ and $\varepsilon'=O(\varepsilon/(s'\log(s'/\delta)))$. For ease of analysis, we assume that each application of \Cref{cor:DT-fixed} uses some fixed number of samples $N_{s'}$ (given by its corresponding sample complexity). This can be done by arbitrarily inflating the number of samples drawn if the query cutoff is reached before the sample cutoff (and likewise for the query counter if the sample cutoff is reached first). Fixing this quantity allows the algorithm to ``know'' what step it is in by checking the sample and query counters.\footnote{More formally, this just means there is a well-defined transition function for the algorithm which relies on the counters to move between potential tree sizes.} 

After each application of \Cref{cor:DT-fixed}, we draw $M_{s'} = O(\log(s'/\delta)/\varepsilon)$ points. If a single point in this sample is un-inferred by the output of \Cref{cor:DT-fixed}, we double $s'$ and continue the process. Otherwise, we output the classifier given by \Cref{cor:DT-fixed} in that round. Notice that tracking the existence of an un-inferred point in this sample takes only a single bit, and moreover that since the sample size of each round is fixed to be $N_{s'} + M_{s'}$, the algorithm can still use the counters to track its position at any step. Finally, we note that for a given $s'$, if the algorithm does not have coverage at least $1-\varepsilon$ it aborts with probability at most $\frac{\delta}{s'}$ by our choice of $M_{s'}$. Since $s'$ starts at $2$ and doubles each round, a union bound gives that the probability the algorithm ever aborts with coverage worse than $1-\varepsilon$ is at most $\delta$ as desired.

It is left to compute the expected memory usage, query complexity, and computational complexity of our algorithm. Notice that at any given step with guess $s'$, our algorithm runs in time
\[
O\left(\frac{d^2s'^2\log(ds'/(\varepsilon\delta))\log(s'/(\varepsilon\delta))}{\varepsilon}\right),
\] 
uses $O(ds')$ memory, and makes at most $O(ds'^2\log(s'/\varepsilon\delta))$ queries. Further, for every round in which $s'\geq s$, the samples drawn are sufficiently large that \Cref{cor:DT-fixed} is guaranteed to succeed with probability at least $1-O(1/s')$ by our chosen parameters. Since the coverage test also succeeds with probability at least $1-O(1/s')$ and each round is independent, for any $s'\geq 16s$ the failure probability at that step is at most $O(1/s'^4)$ (as the algorithm has at this point run at least four rounds with $s' \geq s$). Since all our complexity measures scale like $o(s'^3)$, the expected contribution to time, query complexity, and memory usage at least half in each step for which $s' > 16s$. It is not hard to observe that these final steps then add no additional asymptotic complexity, which gives the desired result.
\end{proof}
\subsection{Halfspaces in Two Dimensions}
Much of the prior work on active learning with enriched queries centers around the class of halfspaces \cite{kane2017active,xu2017noise,hopkins2020noise,hopkins2020power,Cui2020uncertainty}. While it has long been known the class cannot be efficiently active learned in the standard model \cite{dasgupta2005analysis}, KLMZ \cite{kane2017active} showed that adding a natural enriched query called a \textit{comparison} resolves this issue. In particular, recall that a 2D-halfspace is given by the sign of the inner product with some normal vector $v \in S^1$ plus some bias $b \in \mathbb{R}$. A comparison query on two points $x,x' \in \mathbb{R}^2$ measures which point is further from the hyperplane defined by $h$, that is:
\[
\langle x, v \rangle \overset{?}{\geq} \langle x', v \rangle.
\]
KLMZ proved that halfspaces in two dimensions have finite inference dimension, a combinatorial parameter slightly weaker than perfect sample compression that characterizes query-optimal active learning. In fact, we show that 2D-halfspaces with comparisons have substantially stronger structure---namely an LCS of size 5.
\begin{proposition}
The class of halfspaces over $\mathbb{R}^2$ has a size $5$ lossless compression scheme with respect to the comparison oracle.
\end{proposition}
\begin{proof}
Given some (unknown) hyperplane $h=\langle v, \cdot \rangle + b$, monochromatic set $S$, and labels and comparisons $q_h(S)$, we must prove the existence of a subset $T$ of size at most $5$ such that $I(q_h(T)) = I(q_h(S))$. The idea behind our construction is to consider positive rays based on $S$ and $q_h(S)$. In other words, notice that for any point $s \in S$ and pair $(x_1,y_1)$ such that $h(x_1) \geq h(y_1)$, the ray $r_1 = s + t(x_1-y_1)$ is \emph{increasing} with respect to $h$. Furthermore, any two such rays with the same base point $r_1 = s + t(x_1-y_1)$ and $r_2 = s + t(x_2-y_2)$ form a \emph{cone}
\[
C(r_1,r_2) = \left\{ y \in \mathbb{R}^2:~\exists \alpha_1,\alpha_2>0~\text{s.t.}~y = s + \sum\limits_{i=1}^2 \alpha_i(x_i-y_i) \right\}
\]
such that for any $y \in C(r_1,r_2)$, labels and comparisons on the set $\{s,x_1,x_2,y_1,y_2\}$ infer that $h(y) \geq 0$:
\begin{align*}
h(y) &= \left\langle v, s + \sum\limits_{i=1}^2 \alpha_i(x_i-y_i) \right\rangle + b\\
&= \left (\left\langle v, s \right\rangle + b\right) + \sum\limits_{i=1}^2\left\langle v, \alpha_i(x_i-y_i) \right\rangle\\
&= h(s) + \alpha_1(h(x_1)-h(y_1)) + \alpha_2(h(x_2)-h(y_2))\\
&\geq 0.
\end{align*}
Notice that the union of all such cones is itself a cone where the base point $s$ is some minimal element in $S$ (in the sense that for all $s' \in S, h(s) \leq h(s')$), and $r_1$ and $r_2$ are rays stemming from $s$ that make the greatest angle. We argue that $y \in I(q_h(S))$ if and only if $y$ lies inside this cone. Since the cone may be defined by queries on the 5 points making up $r_1$ and $r_2$, this gives the desired compression scheme. See \Cref{fig:halfspace} for a pictorial description of this process.

Denote the cone given by this process by $C$. We have already proved that $y \in I(q_h(S))$ if $y \in C$, so it is sufficient to show that any $y \notin C$ cannot be inferred by queries on $S$. To see this, we examine the bounding hyperplanes of $C$: $s + t(x_1-y_1)$, and $s + t(x_2-y_2)$. Let $H_1$ and $H_2$ denote the corresponding halfspaces (defined such that $C=H_1 \cap H_2$), and define $h_i = \langle v_i, \cdot \rangle + b_i$ such that $H_i=\text{sign}(h_i)$. We first note that the labels and comparisons on $S$ corresponding to each $h_i$ are consistent with those on the true hypothesis $h$. This is obvious for label queries since $S$ is assumed to be entirely positive, and $C$ contains $S$ by definition. For comparisons, assume for the sake of contradiction that there exists a query inconsistent with some $h_i$, that is a pair $x,y \in S$ such that $h(x) \geq h(y)$ but $h_i(x) < h_i(y)$. If this is the case, notice that the ray extending from $y$ through $x$ is decreasing with respect $h_i$, and therefore must eventually cross it. It follows that replacing $(x_i,y_i)$ in the construction of $T$ with $(x,y)$ results in a wider angle, which gives the desired contradiction.

Given that $h_1$ and $h_2$ are both consistent with queries on $S$ under the true hypothesis, consider a point $y$ lying outside of $C$. By definition, either $\text{sign}(h_1(y))$ or $\text{sign}(h_2(y))$ is negative. However, notice that since $h_i = \langle v_i, \cdot \rangle + b_i$ is consistent with queries on $S$, it must also be the case that any non-negative shift $h_{i,b'} = \langle v_i, \cdot \rangle + b_i + b'$ for $b' \ge 0$ is consistent as well. Since for sufficiently large $b'$, $\text{sign}(h_{i,b'}(y))$ is positive, the true label $\text{sign}(h(y))$ cannot be inferred as desired.
\end{proof}
  \begin{figure}
\captionsetup[subfigure]{justification=centering}
    \centering
      \begin{subfigure}{0.32\textwidth}
        \includegraphics[width=\textwidth]{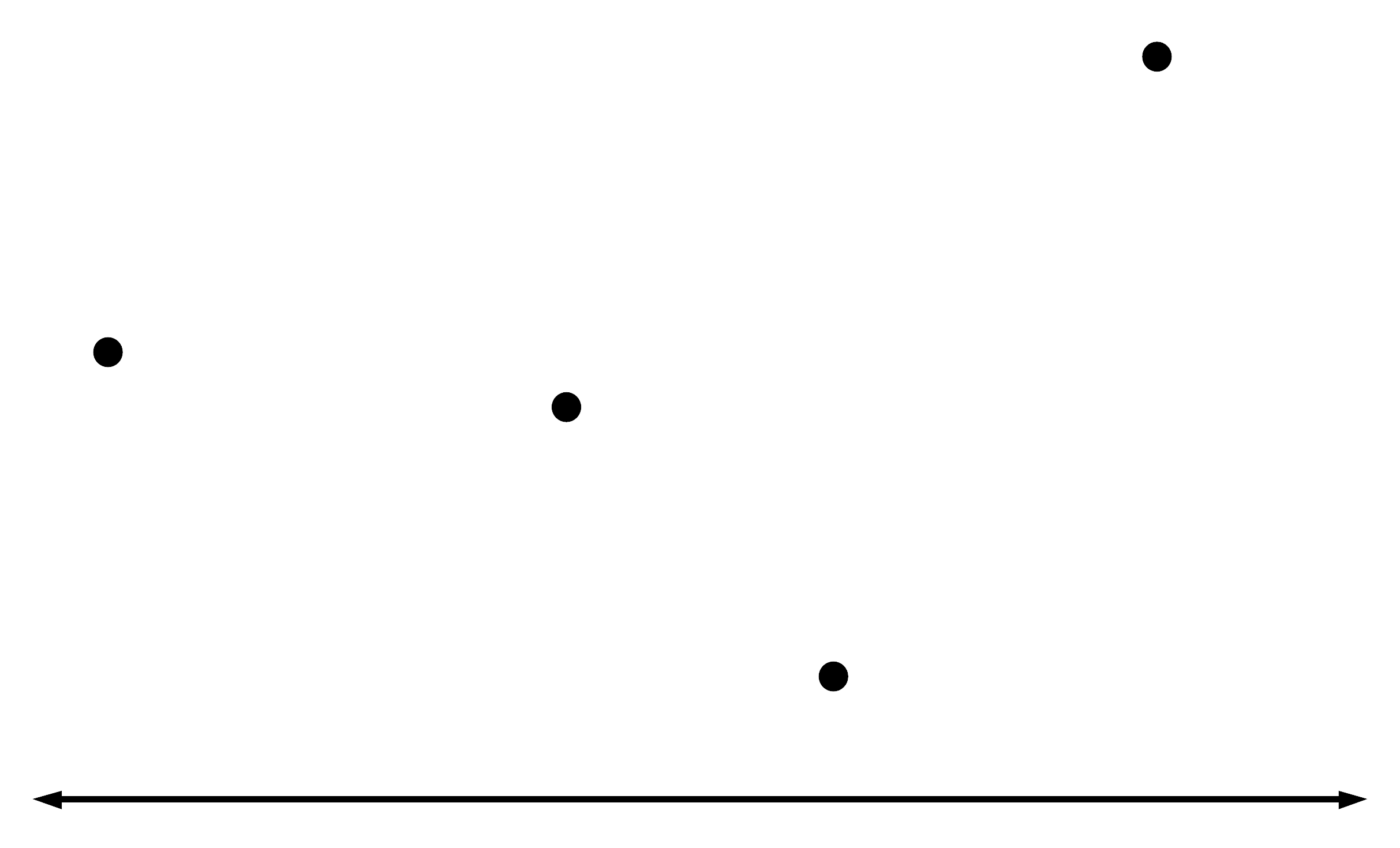}
          \caption{Original sample}
          \label{fig:NiceImage1}
      \end{subfigure}
      \begin{subfigure}{0.32\textwidth}
        \includegraphics[width=\textwidth]{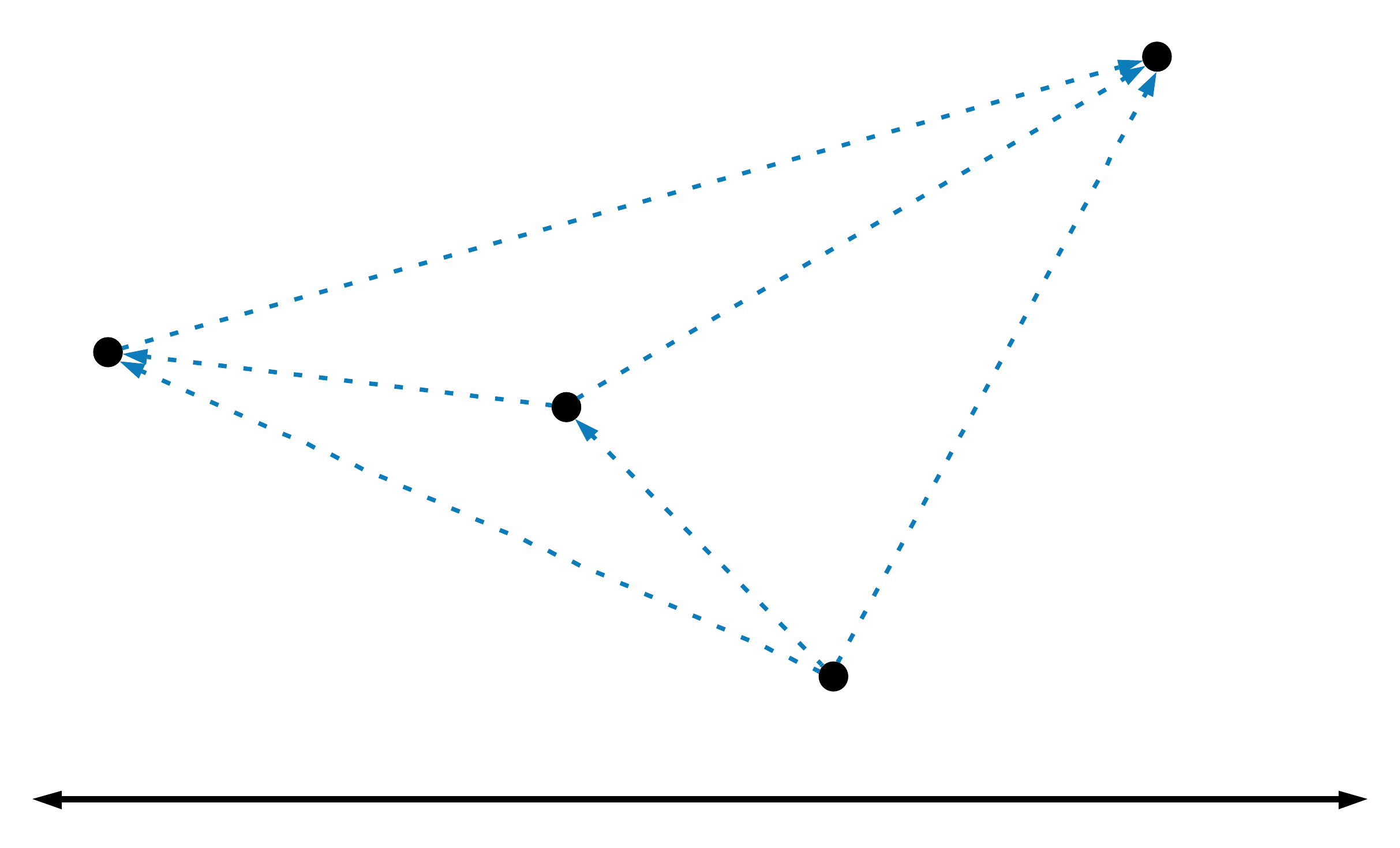}
          \caption{Find increasing directions (blue)}
          \label{fig:NiceImage2}
      \end{subfigure}
      \begin{subfigure}{0.32\textwidth}
        \includegraphics[width=\textwidth]{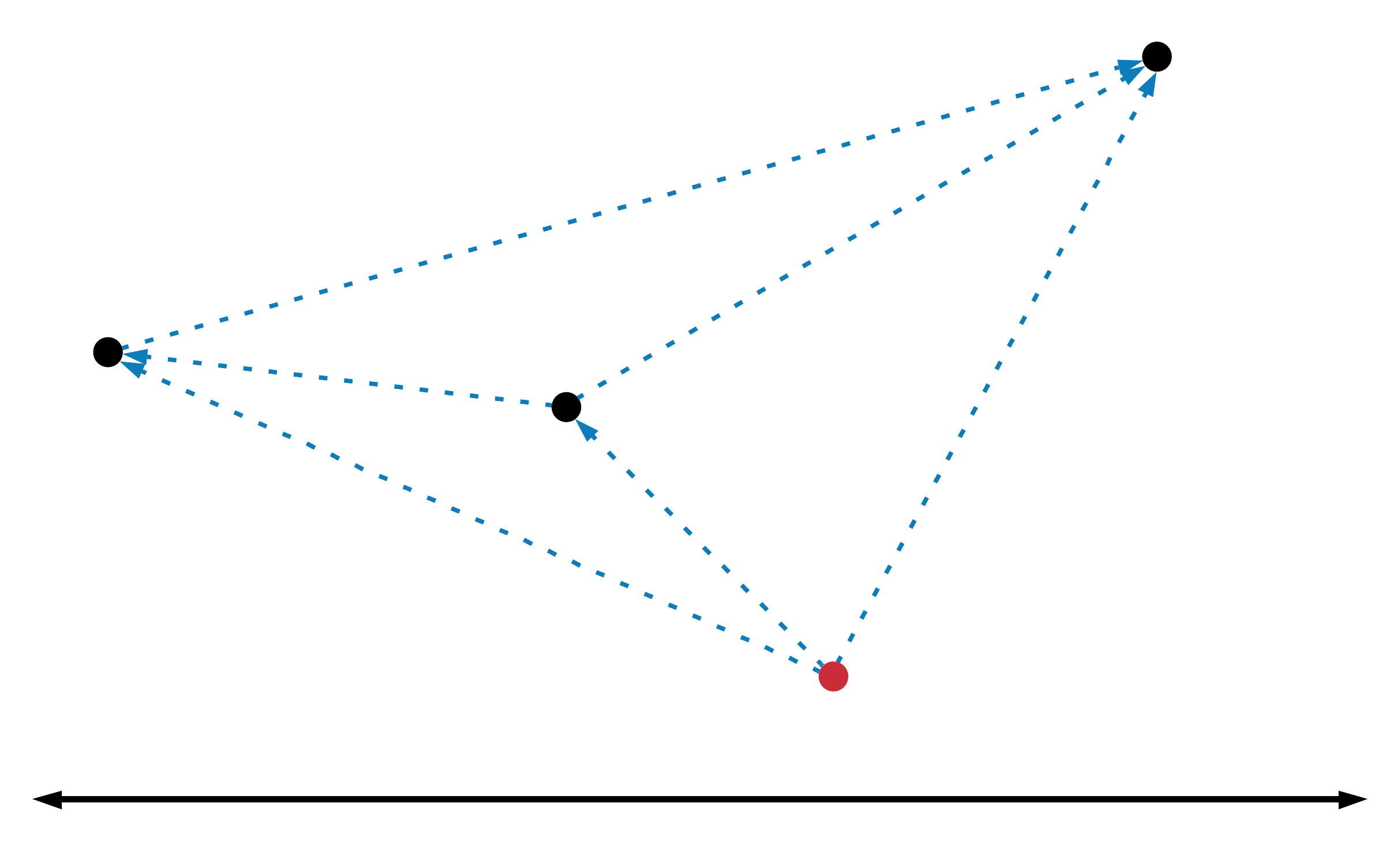}
          \caption{Find minimal point (red)}
      \end{subfigure}
      \\
        \begin{subfigure}{0.32\textwidth}
        \includegraphics[width=\textwidth]{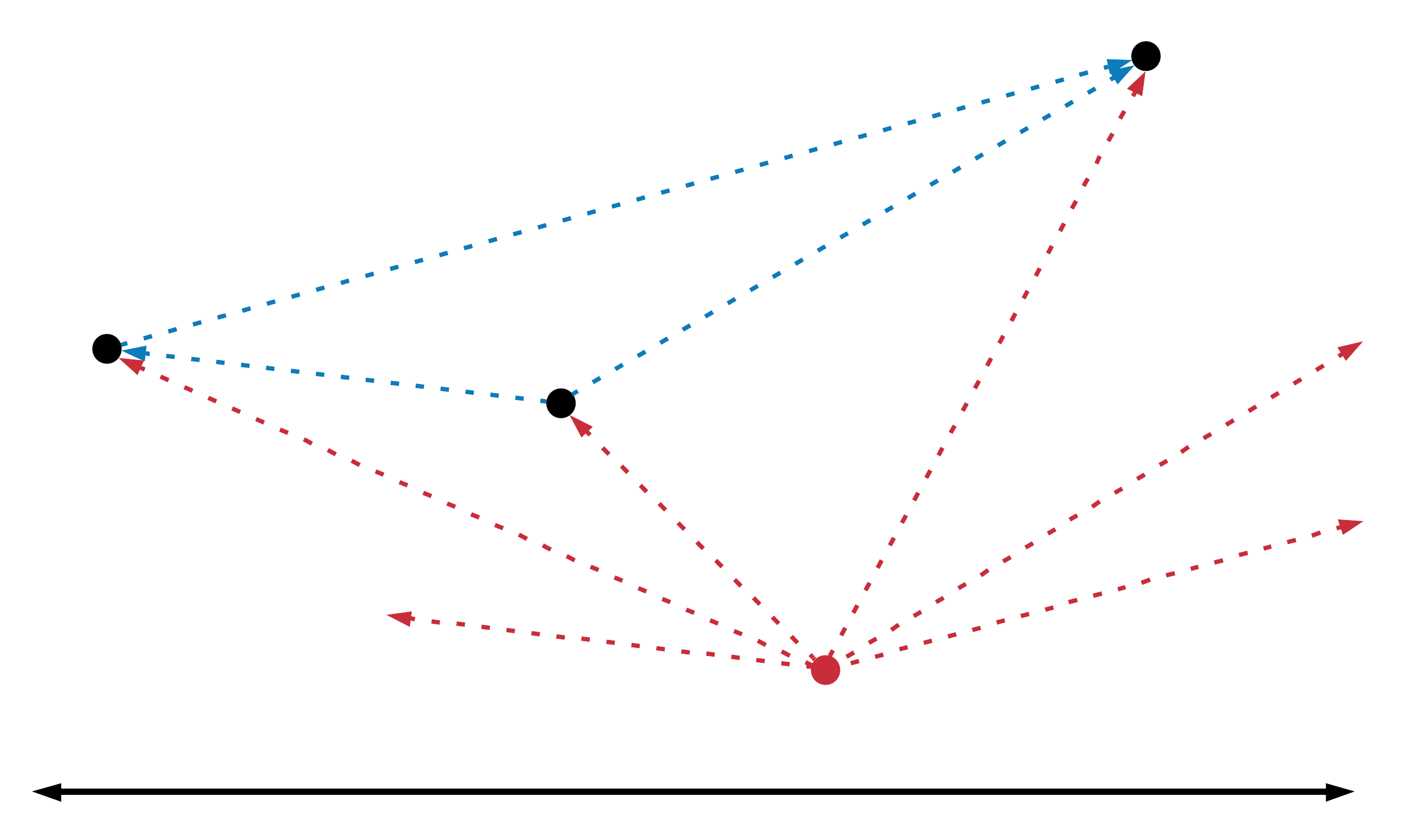}
          \caption{Draw rays from minimum}
      \end{subfigure}
    \begin{subfigure}{0.32\textwidth}
        \includegraphics[width=\textwidth]{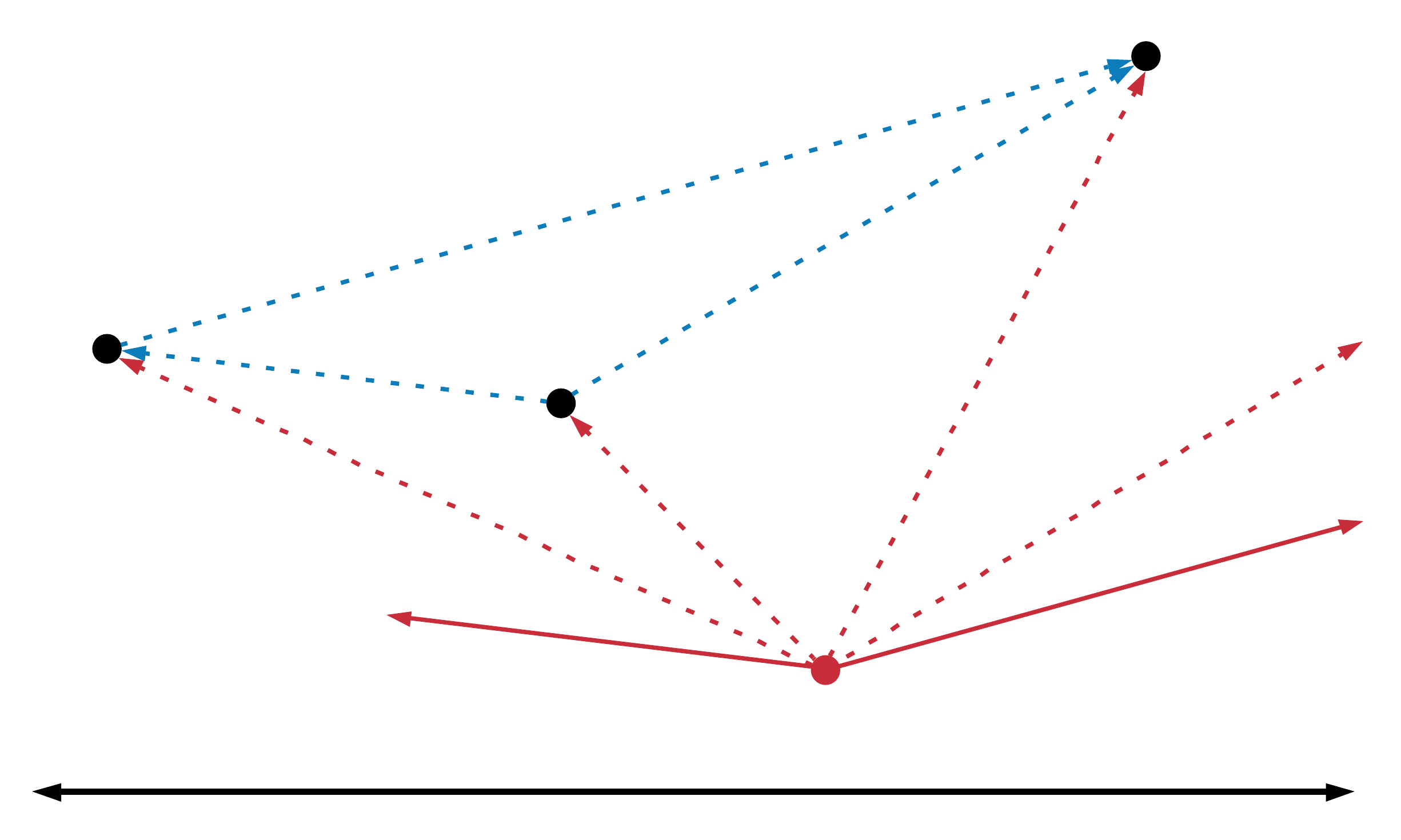}
          \caption{Select widest cone}
      \end{subfigure}
      \caption{A pictorial representation of the intuition behind our LCS for halfspaces. Given a (positive) monochromatic sample (diagram (a)), we find a minimal point (diagram (c)) and all directions of increase (diagram (b)). Combining these gives a set of positive rays which form cones (diagram (d)). Our LCS is given by the points which contribute to the widest cone, depicted in diagram (e) (here all 4 points).}
    \label{fig:halfspace}
\end{figure}
\begin{corollary}\label{cor:halfspace}
The class of halfspaces over $\mathbb{R}^2$ is actively RPU-learnable in only
\[
q(\varepsilon,\delta) \leq O\left(\log\left(\frac{1}{\varepsilon\delta}\right)\right)
\]
queries, $O(1)$ memory, and time $O\left( \frac{\log^2(1/(\delta\varepsilon))}{\varepsilon}\right)$.
\end{corollary}
\begin{proof}
Inference is done through a linear program as in \cite{kane2017active}.
All computational parameters are $O(1)$ due to being in $O(1)$ dimensions, which gives the desired result.
\end{proof}
\section{Further Directions}\label{sec:further}
We end with a brief discussion of two natural directions suggested by our work.

\subsection{Characterizing Bounded Memory Active Learning}
In this work we prove that lossless sample compression is a sufficient condition for efficient, bounded memory active learning in the enriched query regime. KLMZ \cite{kane2017active} prove that inference dimension, a strictly weaker combinatorial parameter for enriched queries, is necessary for efficient bounded memory active learning. Closing the gap between these two conditions remains an open problem---it is currently unknown whether inference dimension is sufficient or lossless sample compression is necessary. Similarly, the relation between inference dimension and lossless sample compression themselves remains unknown. Over finite spaces it is clear from arguments of \cite{kane2017active} that inference dimension and lossless sample compression are equivalent up to a factor of $\log(|X|)$. However, since in the finite regime we are likely interested in learning all points in $|X|$ rather than to some accuracy parameter, the relevant memory bound depends crucially on $\log(|X|)$, making the tightness of the above relation an important question as well.

On a related note, though halfspaces in dimensions greater than two have infinite inference dimension, KLMZ \cite{kane2017active} do show that halfspaces with certain restricted structure (e.g. bounded bit complexity, margin) have finite inference dimension. Whether such classes have lossless compression schemes, or indeed are even learnable in bounded memory at all remains an interesting open problem and a concrete step towards understanding the above.
\subsection{Learning with Noise}
In this work we study only \emph{realizable} case learning. It is reasonable to wonder to what extent our results hold in the \emph{agnostic} model, where the adversary may choose any function (rather than being restricted to one from the concept class $H$), or various models of noise in which the adversary may corrupt queries. While inference-based learning is difficult in such regimes, previous work has seen some success with particular classes of enriched queries such as comparisons \cite{xu2017noise, hopkins2020noise}. Proving the existence of \emph{bounded memory} active learners even for simple noise and query regimes such as random classification noise with comparisons remains an important problem for bringing bounded memory active learning closer to practice.

\bibliographystyle{plainnat}
\bibliography{references.bib} 
\end{document}